\documentclass[nohyperref]{article}

\usepackage{microtype}
\usepackage{graphicx}
\usepackage{subfigure}
\usepackage{booktabs} 

\usepackage{hyperref}
\usepackage{stmaryrd}

\usepackage[accepted]{wfvml2022}

\usepackage{amsmath}
\usepackage{amssymb}
\usepackage{mathtools}
\usepackage{amsthm}

\usepackage[capitalize,noabbrev]{cleveref}

\theoremstyle{plain}
\newtheorem{theorem}{Theorem}[section]

\theoremstyle{definition}

\theoremstyle{remark}
\newtheorem{remark}[theorem]{Remark}

\usepackage[textsize=tiny]{todonotes}

\DeclareSymbolFont{bbold}{U}{bbold}{m}{n}
\DeclareSymbolFontAlphabet{\mathbbold}{bbold}
\newcommand{\vect}[1]{\mathbbold{#1}}
\newcommand{\vectorones}[1][]{\vect{1}_{#1}}

\newcommand{\OF}{\mathsf{F}}
\newcommand{\OG}{\mathsf{G}}

\newcommand{\ON}{\mathsf{N}}

\newcommand{\Lip}{\operatorname{Lip}}

\newcommand{\norm}[2]{\|#1\|_{#2}}

\newcommand{\suchthat}{\;\ifnum\currentgrouptype=16 \middle\fi|\;}
\newcommand{\until}[1]{\{1,\dots, #1\}}
\newcommand{\subscr}[2]{#1_{\textup{#2}}}

\newcommand{\setdef}[2]{\{#1 \; | \; #2\}}

\newcommand{\map}[3]{#1: #2 \rightarrow #3}

\newcommand{\real}{\mathbb{R}}

\newcommand{\diag}{\mathrm{diag}}

\newcommand{\scirc}{\raise1pt\hbox{$\,\scriptstyle\circ\,$}}

\newcommand\oprocendsymbol{\hbox{$\square$}}
\newcommand\oprocend{\relax\ifmmode\else\unskip\hfill\fi\oprocendsymbol}

\wfvmltitlerunning{Robust Training and Verification of Implicit Neural Networks}

\begin{document}

\twocolumn[
\wfvmltitle{Robust Training and Verification of Implicit Neural Networks: \\A Non-Euclidean Contractive Approach
}



\wfvmlsetsymbol{equal}{*}

\begin{wfvmlauthorlist}
\wfvmlauthor{Saber Jafarpour}{equal,yyy}
\wfvmlauthor{Alexander Davydov}{equal,xxx}
\wfvmlauthor{Matthew Abate}{yyy}
\wfvmlauthor{Francesco Bullo}{xxx}
\wfvmlauthor{Samuel Coogan}{yyy}
\end{wfvmlauthorlist}

\wfvmlaffiliation{yyy}{School of Electrical and Computer Engineering, Georgia Institute of Technology, Atlanta, USA, (e-mail: \{saber, matt.abate, sam.coogan\}@gatech.edu)}
\wfvmlaffiliation{xxx}{Center for Control, Dynamical Systems, and Computation, University of California, Santa Barbara, USA, (e-mail: \{davydov, bullo\}@ucsb.edu)}

\wfvmlcorrespondingauthor{Saber Jafarpour}{saber@gatech.edu}

\wfvmlkeywords{Machine Learning, Formal Verification}

\vskip 0.3in
]



\printAffiliationsAndNotice{\wfvmlEqualContribution} 

\begin{abstract}
This paper proposes a theoretical and computational framework for
training and robustness verification of implicit neural networks based upon non-Euclidean contraction theory.
The basic idea is to cast the robustness analysis of a neural network as a reachability problem and use (i) the $\ell_{\infty}$-norm input-output Lipschitz constant and (ii) the tight inclusion function of the network to over-approximate its reachable sets. 
First, for a given implicit neural network, we use $\ell_{\infty}$-matrix measures to propose sufficient conditions for its well-posedness, design an iterative algorithm to compute its fixed points, and provide upper bounds for its $\ell_\infty$-norm input-output Lipschitz constant.
Second, we introduce a related embedded network and show that the embedded network can be used to provide an $\ell_\infty$-norm box over-approximation of the reachable sets of the original network. Moreover, we use the embedded network to design an iterative algorithm for computing the upper bounds of the original system's tight inclusion function. 
Third, we use the upper bounds of the Lipschitz constants and the upper bounds of the tight inclusion functions to design two algorithms for the training and robustness verification of implicit neural networks. 
Finally, we apply our algorithms to train implicit neural networks on the MNIST dataset and compare the robustness of our models with the models trained via existing approaches in the literature. 
\end{abstract}

\section{Introduction}

Recent advances in machine learning have led to increasing deployment
of neural networks in real-world applications, including natural
language processing, computer vision, and self-driving
vehicles. Despite their remarkable computational power, neural networks are notoriously vulnerable to adversarial
attacks; a small perturbations in the input can lead to large
deviations in the
output~\citep{CZ-WZ-IS-JB-DE-IG-RF:13}. Understanding this input
sensitivity is essential in safety-critical applications, since the
consequences of adversarial perturbations can be disastrous. Unfortunately, many of the existing approaches for robustness analysis of neural networks either (i) are based on specific attacks and do not provide any formal guarantees~\citep{AA-LE-AI-KK:18}, or (ii) provide guarantees which are too conservative~\citep{CZ-WZ-IS-JB-DE-IG-RF:13}, or (iii) are not scalable to large-scale problems~\citep{PLC-JCP:20}. As a result, there has been a huge interest in developing computationally tractable and non-conservative algorithms for training and verification of robust neural networks. 

Implicit neural networks are a class of learning
models that replace the explicit hidden layers  with an implicit
equation~\cite{SB-JZK-VK:19,LEG-FG-BT-AA-AYT:21}.
Compared to traditional neural networks, implicit neural networks are known to have
advantages including (i) being more suitable for certain class of learning problems such as
constrained optimization problems~\cite{BA-JZK:17} (ii) being more
memory efficient while maintaining comparable
accuracy~\cite{SB-JZK-VK:19}, and (iii) showing improved
training due to fewer vanishing and exploding
gradients~\cite{AK-ZZ-VS:20}. Despite their benefits, implicit networks can suffer from well-posedness issues and convergence instabilities. Additionally, their input-output behavior may suffer from robustness issues and adversarial perturbations. We note that many of the classical robustness analysis tools for traditional neural networks are either not applicable to implicit neural networks or will lead to conservative results. Indeed, robustness of implicit neural networks is not yet well understood and open questions remain regarding their robust training and verification.

Most of the existing approaches for studying robustness of neural networks focus on either the $\ell_2$-norm or
$\ell_{\infty}$-norm robustness measures. For neural networks with
high-dimensional inputs and subject to dense perturbations, 
$\ell_2$-norm robustness measures are known to provide overly
conservative estimates of robustness and are less informative than
their $\ell_{\infty}$-norm counterparts. In this paper, we propose a framework based on  contraction theory with respect to non-Euclidean $\ell_{\infty}$-norm to study well-posedness, stability, and robustness of implicit neural networks. To provide robustness guarantees, we over-approximate reachable set of implicit neural networks using (i) $\ell_{\infty}$-norm input-output Lipschitz constants, and (ii) input-output tight inclusion functions. We note that, in general, finding the Lipschitz constants and tight inclusion functions of implicit neural networks can be computationally challenging. Using our non-Euclidean contractive approach, we provide non-conservative and computationally tractable estimates of the $\ell_{\infty}$-norm input-output Lipschitz constants and the tight inclusion functions of implicit neural networks. We then use these estimates of the Lipschitz constants and the inclusion functions to design two algorithms for training and verification of implicit neural networks with respect to $\ell_{\infty}$-box input perturbations.  Finally, we evaluate the performance and efficiency of our algorithms for training robust implicit neural networks on the MNIST dataset. This paper is a review of the accepted
papers~\citep{SJ-AD-AVP-FB:21f,SJ-MA-AD-FB-SC:21y} and the submitted
paper~\citep{AD-SJ-MA-FB-SC:21z}.


\section{Related works}

\paragraph{Robustness of neural networks.}
Starting with~\citep{IJG-JS-CZ:15}, a large body of research has focused on the design of neural networks that are robust with respect to adversarial perturbations~\citep{NP-PM-XW-SJ-AS:16}. Unfortunately, many of these approaches are based on robustness with respect to specific attacks and they do not provide formal robustness guarantees~\citep{AA-LE-AI-KK:18}. Recent research has focused on 
providing provable robustness guarantees for
neural networks~\citep{AM-AM-LS-DT-AV:17,NC-DW:17}. Rigorous methods for training and/or verifying neural networks generally fall into four different categories (i) Lipschitz bound
methods~\citep{MF-AR-HH-MM-GJP:19,AV-KS:18,PLC-JCP:20}, (ii) interval
bound methods~\citep{MM-TG-MV:18,SG-etal:18,HZ-etal:20}, (iii)
optimization-based methods~\citep{EW-ZK:18,HZ-etal:18}, and (iv)
probabilistic methods~\citep{JC-ER-JZK:19,BL-CC-WW-LC:19}. 

\paragraph{Implicit learning models.}

Several frameworks for studying implicit models of learning have been developed in the literature~\citep{SB-JZK-VK:19,LEG-FG-BT-AA-AYT:21}. Regarding the well-posedness of implicit neural networks, \citep{LEG-FG-BT-AA-AYT:21} proposes a sufficient spectral condition for existence of solutions for the fixed point
equation. In~\citep{EW-JZK:20,MR-RW-IRM:20}, using monotone operator
theory, a suitable parametrization of the weight matrix is proposed
which guarantees the stable convergence of suitable fixed point
iterations.  The work \citep{SJ-AD-AVP-FB:21f} proposes non-Euclidean
contraction theory to design implicit neural networks and study their
well-posedness, stability, and robustness with respect to the
$\ell_\infty$-norm. Regarding the robustness guarantees, compared to the traditional neural networks, there are far fewer works on the robust verification and training of implicit neural network. In~\citep{LEG-FG-BT-AA-AYT:21} a sensitivity-based
robustness analysis for implicit neural network is proposed. Approximation of the Lipschitz constants of deep equilibrium
networks has been studied in~\citep{CP-EW-JZK:21,MR-RW-IRM:20}. Recently, the ellipsoid methods based on semi-definite programming~\citep{TC-JBL-VM-EP:21} and the interval-bound propagation method~\citep{CW-JZK:22} have been proposed for robustness
certification of deep equilibrium networks.

\section{Mathematical Preliminaries}

\paragraph{Matrices and functions.} Given a matrix $B \in \mathbb{R}^{n\times m}$, we denote the
non-negative part of $B$ by $[B]^+ = \max(B, 0)$ and the nonpositive
part of $B$ by $[B]^- = \min(B, 0)$. The Metzler part and the
non-Metzler part of square matrix $A\in \real^{n\times n}$ are
denoted by $\lceil A \rceil^{\mathrm{Mzl}}\in \real^{n\times n}$ and
$\lfloor A \rfloor^{\mathrm{Mzl}}\in \real^{n\times n}$, respectively,
where
\begin{align*}
  (\lceil A \rceil^{\mathrm{Mzl}})_{ij} &=\begin{cases}
    A_{ij} & A_{ij} \geq 0\; \mbox{or} \; i =  j\\
    0 & \mbox{otherwise,}
  \end{cases}\\ \lfloor A\rfloor^{\mathrm{Mzl}}&= A-\lceil A
        \rceil^{\mathrm{Mzl}}.
\end{align*}
For matrices $C\in \real^{n\times m}$ and
$D\in \real^{p\times q}$, the Kronecker product of $C$ and $D$ is
denoted by $C\otimes D$. For every $\eta\in \real^n_{>0}$, we denote the largest (smallest) component of $\eta$ by $\eta_{\max}$ ($\eta_{\min}$). Moreover, we define
the diagonal matrix $[\eta]\in \real^{n\times n}$ by
$[\eta]_{ii}=\eta_i$, for every $i\in \{1,\ldots,n\}$. For
$\eta\in \real^n_{>0}$, the diagonally weighted $\ell_{\infty}$-norm
is defined by $\|x\|_{\infty,[\eta]^{-1}}=\max_i|x_i|/\eta_i$, the
diagonally weighted $\ell_{\infty}$-matrix measure is defined by
$\mu_{\infty,[\eta]^{-1}}(A) =\max_{i \in \until{n}} A_{ii} + \sum_{j
  \neq i} \frac{\eta_j}{\eta_i} |A_{ij}|$. The $\ell_2$-matrix measure is also defined by $\mu_2(A) = \frac{1}{2}\lambda_{\max}(A+A^{\top})$, where $\lambda_{\max}$ denoted the largest eigenvalue. Let $f:\real^r\to \real^q$
be a locally Lipschitz map and $\mathcal{X}\subseteq \real^r$. The $\ell_{\infty}$-norm 
Lipschitz constant of $f$ over $\mathcal{X}$ is the smallest real
number $\Lip_{\infty}^{\mathcal{X}}(f)\in \real_{\ge 0}$ such that
\begin{align*}
  \|f(x)-f(y)\|_{\infty} \le \Lip_{\infty}^{\mathcal{X}}(f) \|x-y\|_{\infty}
\end{align*}
for every $x,y\in \mathcal{X}$. We denote the $\ell_{\infty}$-norm Lipschitz constant of
$f$ over $\real^r$ by $\Lip_{\infty}(f)$. Let $\OF:\real^n\times \real^m\to \real^n$ be a mapping, for every $\alpha\in (0,1]$, we define the $\alpha$-average map
$\OF_{\alpha}:\real^n\times \real^m\to \real^n$ by
$\OF_{\alpha}(x,u) = (1-\alpha)x+ \alpha \OF(x,u)$, for every $x\in \real^n$ and every $u\in \real^m$. 

\paragraph{Intervals and inclusion functions.} For every $x\le
\widehat{x}$, we define the interval $[x,\widehat{x}]=\setdef{y\in
  \real^n}{x\le y\le\widehat{x}}$. The subset $\mathcal{T}^n\subset
\real^{2n}$ is defined by $\mathcal{T}^{n} :=
\setdef{(x,\widehat{x})\in \real^{2n}}{x\le \widehat{x}}$.  Let
$f:\real^{r}\to \real^q$ be a mapping. Then $\OF=
\left[\begin{smallmatrix}\underline{\OF}\\ \overline{\OF}\end{smallmatrix}\right]:\mathcal{T}^r\to \real^{2q}$ is  an inclusion function for $f$, if, for every $x\le y \le \widehat{x}$, 
\begin{enumerate}
    \item  $\underline{\OF}(y,y)\ge \underline{\OF}(x,\widehat{x})$ and $\overline{\OF}(y,y)\le \overline{\OF}(x,\widehat{x})$; 
    \item $\underline{\OF}(x,x)=\overline{\OF}(x,x)=f(x)$. 
\end{enumerate}
If $\OF$ is an inclusion function for $f$, then it is easy to see that,
\begin{align}\label{eq:interval-bound}
    f([x,\widehat{x}])\subseteq [\underline{\OF}(x,\widehat{x}),\overline{\OF}(x,\widehat{x})],\quad \mbox{for all }x\le \widehat{x}.
\end{align}

\section{Reachability analysis of learning models}\label{sec:reach}

Given a nonlinear learning model with the input-output map
$f:\real^{r}\to \real^{q}$ and the input set
$\mathcal{X}\subseteq \real^r$, the reachable set of $f$ is given by
\begin{align*}
  \mathcal{Y} = f(\mathcal{X}) = \setdef{y\in \real^{q}}{y=f(x), x\in \mathcal{X}}
  \end{align*}
Many desirable properties of the learning model, such as robustness
and safety, can be presented as $\mathcal{Y}$ belonging to a specification set $S\subset \real^q$. However, verification of these specifications requires an exact computation
of the set $\mathcal{Y}$ which is usually complicated. In this section, we review two frameworks for
over-approximating the reachable sets of $f$ and study the connection
between these two settings.

\paragraph{Lipschitz constants.}
For the nonlinear learning model $f:\real^r\to \real^q$, the $\ell_{\infty}$-norm Lipschitz
constant $\Lip_{\infty}(f)$ provide the tightest $\ell_\infty$-norm
over-approximation of reachable set of $f$. We define the set
$\Omega =\setdef{x\in \real^r}{\frac{\partial f}{\partial x} \;\;
  \mbox{exists}}$. By Rademacher's theorem, the set $\real^r/\Omega$
is a measure zero set. As a result, one can compute the $\ell_{\infty}$-norm Lipschitz constant of $f$ using the following optimization problem:
\begin{align}\label{eq:tight-lip}
  \Lip_{\infty} (f) = \sup_{x\in \Omega} \|D f(x)\|_{\infty}
\end{align}

\paragraph{Inclusion functions.} The mapping $\OF=
\left[\begin{smallmatrix}\underline{\OF}\\
    \overline{\OF}\end{smallmatrix}\right]:\mathcal{T}^r\to
\real^{2q}$ is \emph{tight inclusion function} for $f$, if, for every other inclusion function $\OG= 
    \left[\begin{smallmatrix}\underline{\OG}\\ \overline{\OG}\end{smallmatrix}\right]:\mathcal{T}^r\to \real^{2q}$ of $f$, we have
    \begin{align*}
        \underline{\OG}(x,\widehat{x})\le \underline{\OF}(x,\widehat{x}),\quad
        \overline{\OF}(x,\widehat{x})\ge \overline{\OG}(x,\widehat{x}),\quad\mbox{for all  }x\le \widehat{x}
    \end{align*} 
The tight inclusion function $\OF$ provides the tightest box
over-approximation of reachable sets of $f$. It is easy to see that if $\OF =\left[\begin{smallmatrix}\underline{\OF}\\ \overline{\OF}\end{smallmatrix}\right]$ is the tight inclusion function for $f$, then it can be computed using the following optimization problem, for every $i\in \{1,\ldots,n\}$:
\begin{align}\label{eq:tight-inclusion}
    \underline{\OF}_i(x,\widehat{x})= \min_{z\in [x,\widehat{x}]} f_i(z),\quad
    \overline{\OF}_i(x,\widehat{x})=\max_{z\in [x, \widehat{x}]} f_i(z)
\end{align}

The next theorem shows that, compared to Lipschitz constants, tight
inclusion functions provide sharper estimates of reachable sets. 
   
\begin{theorem}[Inclusion function vs. Lipschitz constant]\label{thm:LipInc}
Let $f:\real^r\to \real^q$ be a continuous mapping and $\OF=\left[\begin{smallmatrix}\underline{\OF}\\ \overline{\OF}\end{smallmatrix}\right]:\mathcal{T}^r\to \real^{2q}$ be the tight inclusion function for $f$. Then, for every $\underline{x}\le \overline{x}$, we have 
\begin{equation*}
    \|\overline{\OF}(\underline{x},\overline{x})-\underline{\OF}(\underline{x},\overline{x})\|_{\infty} \le \Lip^{[\underline{x},\overline{x}]}_{\infty}(f) \|\underline{x}-\overline{x}\|_{\infty}.
\end{equation*}
\end{theorem}
\begin{proof} Let $i\in \{1,\ldots,k\}$ be such that $\|\overline{\OF}(\underline{x},\overline{x})-\underline{\OF}(\underline{x},\overline{x})\|_{\infty} = \left|\overline{\OF}_i(\underline{x},\overline{x})-\underline{\OF}_i(\underline{x},\overline{x})\right|$. 
Note that since $f$ is continuous and the box $[\underline{x},\overline{x}]$ is compact, there exist $\eta^*,\xi^*\in [\underline{x},\overline{x}]$ such that 
\begin{align*}
    \max_{y\in [\underline{x},\overline{x}]}f_i(y) =f_i(\eta^*), \qquad\min_{y\in [\underline{x},\overline{x}]}f_i(y)=f_i(\xi^*).
\end{align*}
This implies that $\|\overline{\OF}(\underline{x},\overline{x})-\underline{\OF}(\underline{x},\overline{x})\|_{\infty}=
  |f_i(\eta^*)-f_i(\xi^*)| = \|f(\xi^*)-f(\eta^*)\|_{\infty} $ and thus
\begin{align*}
  \|\overline{\OF}(\underline{x},\overline{x})-\underline{\OF}(\underline{x},\overline{x})\|_{\infty} &\le 
   \Lip^{[\underline{x},\overline{x}]}_{\infty}\|\xi^*-\eta^*\|_{\infty}
   \\ &\le \Lip^{[\underline{x},\overline{x}]}_{\infty}(f) \|\underline{x}-\overline{x}\|_{\infty}.
\end{align*}
\end{proof}
\vspace{-1cm}

\section{Implicit neural networks}

Given $W\in\real^{n\times{n}}$, $U\in\real^{n\times{r}}$, $b\in \real^{n}$
$C\in\real^{q\times{n}}$, and $c\in\real^{q}$, we consider
the implicit neural network
\begin{align}
  \label{eq:inn}
  z&=\Phi(Wz+Ux+b) := \ON(z,x)\nonumber, \\  y &= Cz + c,
\end{align}
where $z\in\real^n$, $x\in\real^r$, $y\in\real^{q}$, and $\map{\Phi}{\real^n}{\real^n}$
is defined by $\Phi(z) =
(\phi_1(z_1),\ldots,\phi_n(z_n))$.
For every $i\in \{1,\ldots,n\}$, we assume the activation function
$\map{\phi_i}{\real}{\real}$ is weakly increasing, i.e., 
$\phi_i(x_i)\geq\phi_i(z_i)$ for $x_i\geq z_i$, and non-expansive, i.e.,
$|\phi_i(x_i)-\phi_i(z_i)|\leq|x_i-z_i|$ for all $x_i$ and $z_i$; if
$\phi_i$ is differentiable, these conditions are equivalent to
$0\leq\phi_i'(x_i)\leq1$ for all $x_i\in\real$. The above
two assumptions holds for a large class of activation function
including but not limited to ReLU, leakyReLU, tanh, and sigmoid
functions~\citep{LEG-FG-BT-AA-AYT:21}. It is known that implicit neural networks can be ill-posed and can suffer from convergence instability. The next theorem provides a sufficient condition for well-posedness of the implicit neural network~\eqref{eq:inn}. We refer
to~\citep[Theorem 3]{SJ-AD-AVP-FB:21f} for the proof. 

\begin{theorem}[Well-posedness and computation of fixed points]
  \label{thm:inn-well-posed}
 Consider the implicit neural network~\eqref{eq:inn}. Given a vector $\eta\in \real^n_{>0}$, if
 $\mu_{\infty,[\eta]^{-1}}(W)<1$ holds, then the following statements
 are true:
 \begin{enumerate}
  \item\label{fact:wellposed} the fixed point equation~\eqref{eq:inn} is well-posed, i.e., there exists a unique fixed
    point,

  \item\label{p2:average} for every $\alpha\in (0,
 \big(1-\min_{i\in\until{n}}[W_{ii}]^-\big)^{-1}]$, the
 $\alpha$-average iteration $z^{k+1}=\ON_\alpha(z^k)$ is contracting with
    respect to the norm $\|\cdot\|_{\infty,[\eta]^{-1}}$ and converges
    to the unique fixed point of the implicit neural networks~\eqref{eq:inn}. 
  \end{enumerate}
\end{theorem}
\begin{remark}[Comparison to the literature]
  \begin{enumerate}
  \item In~\citep{LEG-FG-BT-AA-AYT:21} a well-posedness condition of
    the form $\lambda_{\mathrm{pf}}(|W|)<1$ is proposed, where $|W|$
    denotes the entrywise absolute value of the matrix $W$ and
    $\lambda_{\mathrm{pf}}$ denotes the Perron-Frobenius
    eigenvalue. Our well-posedness condition in
    Theorem~\ref{thm:inn-well-posed}\eqref{fact:wellposed} is less
    conservative than the condition $\lambda_{\mathrm{pf}}(|W|)<1$ and
    its convex relaxation of the form $\|W\|_{\infty}<1$ proposed
    in~\citep{LEG-FG-BT-AA-AYT:21}.
    
  \item In~\citep{EW-JZK:20} a framework based on
    Monotone Operator Theory is developed for studying implicit neural
    network~\eqref{eq:inn} with well-posedness condition
    $I_n-\tfrac{1}{2}(W+W^{\top})\succeq (1-\gamma)I_n$. In the context of contraction theory~\citep{WL-JJES:98},
    \begin{equation*}
      I_n-\tfrac{1}{2}(W+W^{\top})\succeq (1-\gamma)I_n \; \iff \; \mu_2(W)\le \gamma.
    \end{equation*}
    We refer to~\citep{SJ-AD-AVP-FB:21f} for the proof. Thus, our
    framework can be considered as the non-Euclidean version of the
    setting in~\citep{EW-JZK:20}. 
    
    
  \end{enumerate}
\end{remark}

\section{Robustness of implicit neural networks}

In this section, we study input-output robustness of the implicit
neural network~\eqref{eq:inn} using the reachability frameworks of
Section~\ref{sec:reach} for the input-output map $f_{\ON}:\real^{r}\to \real^{q}$:
\begin{align}\label{eq:input-output}
f_{\ON}(x) : = C z^*_x + c   
\end{align}
where $z^*_x\in \real^n$ satisfies $z^*_x = \ON(z^*_x,x)$. 

\paragraph{Robustness via Lipschitz bounds.} We use
the Lipschitz constant framework in Section~\ref{sec:reach} to study reachable
sets of implicit neural networks. Finding the Lipschitz constant of the input-output map $f_{\ON}$
requires solving the optimization problem~\eqref{eq:tight-lip} which
is computationally intractable for large-scale networks. In
the next theorem, we provide an upper bound for this Lipschitz
constant and use it to over-approximate the
reachable set of the neural network. We refer to~\cite{SJ-AD-AVP-FB:21f} for the proof.

\begin{theorem}[Input-output Lipschitz bounds]\label{thm:lip}
  Consider the implicit neural network~\eqref{eq:inn} with the
  input-output map $f_{\ON}$ defined in~\eqref{eq:input-output}. Let $\eta\in \real^n_{>0}$ be such that
  $\mu_{\infty,[\eta]^{-1}}(W)<1$:

  \begin{enumerate}
      \item\label{p1:bound} the Lipschitz constant
        of $f_{\ON}$ is bounded by:
        \begin{align*}
          \Lip_{\infty}(f_{\ON})\le \left(\frac{\eta_{\max}}{\eta_{\min}}\right)\frac{\norm{U}{\infty}
        \norm{C}{\infty}}
          {1-\mu_{\infty,[\eta]^{-1}}(W)^+}
          \end{align*}
      \item\label{p2:interval} for every $x,x'\in \real^r$, by denoting $\xi:=\left(\tfrac{\eta_{\max}}{\eta_{\min}}\right)\frac{\norm{U}{\infty}
        \norm{C}{\infty} }
          {1-\mu_{\infty,[\eta]^{-1}}(W)^+}\|x-x'\|_{\infty} \in
          \real_{>0}$, we have
        \begin{align*}
         f_{\ON}(x') \in 
        [f_{\ON}(x)-\xi \vect{1}_q, f_{\ON}(x)
          +\xi \vect{1}_q]. 
          \end{align*}
    \end{enumerate}
  \end{theorem}

  \begin{remark}[Comparison with the literature]
    \begin{enumerate}
  \item In~\citep{LEG-FG-BT-AA-AYT:21}, the following upper bound for the tight $\ell_{\infty}$-norm
    input-output Lipschitz constant of the implicit neural
    network~\eqref{eq:inn} is obtained:
    \begin{align*}
      \Lip_{\infty}(f_{\ON}) \le \frac{\|U\|_{\infty}\|C\|_{\infty}}{1-\|W\|_{\infty}}
    \end{align*}
    Since $\mu_{\infty}(W) \le \|W\|_{\infty}$, for every $W\in
    \real^{n\times n}$, we can conclude that, compared to~\citep{LEG-FG-BT-AA-AYT:21}, Theorem~\ref{thm:lip}\eqref{p1:bound}
    provides a sharper estimate for the Lipschitz constant of
    implicit neural network~\eqref{eq:inn}

  \item In~\citep{CP-EW-JZK:21} upper bounds for the $\ell_2$-norm input-output
    Lipschitz constant of 
    implicit neural network~\eqref{eq:inn} are obtained. However, these estimates
    are restricted to implicit neural networks with ReLU activation
    functions and cannot be extended to more general classes of
    activation functions.
  \end{enumerate}
   \end{remark}

  \paragraph{Robustness via inclusion functions.} Next, we use the
  inclusion function framework in Section~\ref{sec:reach} to study
  reachable sets of the implicit neural network~\eqref{eq:inn}. Finding tight inclusion
  function of the input-output map $f_{\ON}$ requires solving the
  optimization problem~\eqref{eq:tight-inclusion} which is 
  computationally intractable for large-scale networks. In this
  section, we provide an upper
  bound for this tight inclusion function.

  We first introduce the
  embedded implicit neural network associated
  with~\eqref{eq:inn}. Given $\underline{x}\le\overline{x}$ in
  $\real^r$, we define the \emph{embedded implicit neural network} by
  \begin{align}\label{eq:INN-embedding}
        \begin{bmatrix}\underline{x}\\\overline{x}\end{bmatrix}
        &= \begin{bmatrix}\Phi(\lceil W \rceil^{\mathrm{Mzl}} \underline{z}+\lfloor W \rfloor^{\mathrm{Mzl}} \overline{z} +
          [U]^{+}\underline{x} + [U]^{-}\overline{x} + b)\\ \Phi(\lceil W \rceil^{\mathrm{Mzl}}\overline{z}+\lfloor W \rfloor^{\mathrm{Mzl}}
          \underline{z} + [U]^{+}\overline{x} + [U]^{-}\underline{x}+b)\end{bmatrix}\nonumber\\
        \begin{bmatrix}\underline{y}\\\overline{y}\end{bmatrix} &=
        \begin{bmatrix}[C]^+ & [C]^- \\ [C]^- & [C]^+\end{bmatrix}  \begin{bmatrix}\underline{x}\\\overline{x}\end{bmatrix} +
        \begin{bmatrix}c\\ c\end{bmatrix}. 
  \end{align}
  We also define $\ON^{\mathrm{E}}:\real^{2n}\times \real^{2r}\to \real^n$ by
  \begin{multline*}
    \ON^{\mathrm{E}}(z,\widehat{z},x,\widehat{x}) := \\ \Phi(\lceil W \rceil^{\mathrm{Mzl}} z+\lfloor W \rfloor^{\mathrm{Mzl}} \widehat{z} +
          [U]^{+}x + [U]^{-}\widehat{x} + b).
    \end{multline*}
      Intuitively, the embedded implicit neural
      network~\eqref{eq:INN-embedding} is an implicit neural network with the
      box input $[\underline{x}, \overline{x}]$ and the box output
      $[\underline{y},\overline{y}]$. Surprisingly, one can show the sufficient
      condition for well-posedness of the implicit neural
      network~\eqref{eq:inn} in Theorem~\ref{thm:inn-well-posed} is also a sufficient condition for well-posedness of the embedded
      implicit neural network~\eqref{eq:INN-embedding}. In the next
      theorem, we study the connection between robustness of the implicit neural
      network~\eqref{eq:inn} and  reachability of the embedded implicit neural
      network~\eqref{eq:INN-embedding}. We refer to~\citep[Theorem 1]{SJ-MA-AD-FB-SC:21y} for the proof. 



\begin{theorem}[Input-output inclusion function]\label{thm:INN}
  Consider the implicit neural network~\eqref{eq:inn}. Let $\eta\in \real^n_{>0}$ is such that $\mu_{\infty,[\eta]^-1}(W)<1$. Then, for $\alpha\in
  (0, (1-\min_{i\in \{1,\ldots,n\}}(W_{ii})^{-})^{-1}]$, the following
  statements hold:
  \begin{enumerate}
  \item\label{p1:L4DC} the $\alpha$-average iterations
    $ \left[\begin{smallmatrix}\underline{z}^{k+1} \\
        \overline{z}^{k+1}\end{smallmatrix}\right] =
    \left[\begin{smallmatrix}\ON^{\mathrm{E}}_{\alpha}(\underline{z}^k,\overline{z}^k,\underline{x},\overline{x})\\
        \ON^{\mathrm{E}}_{\alpha}(\overline{z}^k,\underline{z}^k,\overline{x},\underline{x})\end{smallmatrix}\right]$
    is contracting with respect to the norm
    $\|\cdot\|_{\infty,I_2\otimes [\eta]^{-1}}$ and converge to the
    unique fixed point $\begin{bmatrix}
      \underline{z}^*\\
      \overline{z}^*
    \end{bmatrix}$ of the embedded implicit neural network~\eqref{eq:INN-embedding};
  \item\label{p2:L4DC} the $\alpha$-average iterations $z^{k+1}=\ON_{\alpha}(z^k,x)$
    is contracting with respect to the norm
    $\|\cdot\|_{\infty,[\eta]^{-1}}$ and converges to the unique fixed
    point $z^*\in [\underline{z}^*,\overline{z}^*]$ of the
    implicit neural network~\eqref{eq:inn};
  \item\label{p3:L4DC} for the tight inclusion function $\OF_{\ON}=\left[\begin{smallmatrix}
        \underline{\OF}_{\ON}\\
        \overline{\OF}_{\ON}
          \end{smallmatrix}\right]:\mathcal{T}^{r}\to \real^{2q}$ of
        the input-output map $f_{\ON}$ defined by equation~\eqref{eq:input-output}, we have 
          \begin{align*}
            \underline{\OF}_{\ON}(\underline{x},\overline{x}) &\ge
                                                                [C]^{+}\underline{z}^*+[C]^{-}\overline{z}^*+c :=
             \underline{\OG}_{\ON}(\underline{x},\overline{x})\\
            \overline{\OF}_{\ON}(\underline{x},\overline{x}) &\le
                                                               [C]^{+}\overline{z}^*+[C]^{-}\underline{z}^*+c
                                                               := \overline{\OG}_{\ON}(\underline{x},\overline{x})
          \end{align*}
        \item for every $x\in [\underline{x},\overline{x}]$, we have
          \begin{align*}
            f_{\ON}(x)\in
            [\underline{\OG}_{\ON}(\underline{x},\overline{x}), \overline{\OG}_{\ON}(\underline{x},\overline{x})].
            \end{align*}         
        \end{enumerate}
      \end{theorem}
      \begin{remark}
        \begin{enumerate}
        \item Theorem~\ref{thm:INN} (resp. Theorem~\ref{thm:inn-well-posed}) can be interpreted as a dynamical
          system approach to study robustness (resp. well-posedness) of implicit neural
          networks. Indeed, it is easy to see that the
          $\alpha$-average map $\ON^{\mathrm{E}}_{\alpha}$
          (resp. $\ON_{\alpha}$) is the forward Euler
          discretization of the dynamical system
          $\frac{d}{dt}\left[\begin{smallmatrix}\underline{x}\\
            \overline{x}\end{smallmatrix}\right] =
          -\left[\begin{smallmatrix}\underline{x}\\ \overline{x}\end{smallmatrix}\right] +
          \left[\begin{smallmatrix}
          \ON^{\mathrm{E}}(\underline{x},\overline{x},\underline{u},\overline{u})\\
          \ON^{\mathrm{E}}(\overline{x},\underline{x},\overline{u},\underline{u})
        \end{smallmatrix}\right]$ (resp. $\frac{dx}{dt}=-x+\ON(x,u)$). it is easy to see that the condition $\mu_{\infty,[\eta]^{-1}}(W)<1$ ensures that these dynamical systems are contracting with respect to $\|\cdot\|_{\infty,[\eta]^{-1}}$~\citep{WL-JJES:98}.

      \item In terms of evaluation time, computing the
        $\ell_{\infty}$-norm box bounds on the output is equivalent to
        two forward passes of the original implicit network.
      
      \item It can be shown that Implicit neural networks contain feedforward neural
        networks as a special
        case~\citep{LEG-FG-BT-AA-AYT:21}. Indeed, for a fully-connected feedforward
        neural network with $k$ layers and $n$ neurons in each layer,
        there exists an implicit network representation with block
        upper diagonal weight matrix $W\in \real^{kn\times
          kn}$~\citep[Section 3.2]{LEG-FG-BT-AA-AYT:21}. As a result, for small enough $\delta>0$ and
        for $\eta = (\delta, \delta^2 \ldots, \delta^{k})^{\top}\in
        \real^{k}_{>0}$, we have $\mu_{\infty,[\eta]^{-1}\otimes I_n}(W)<1$. In this case, the fixed point
        of the embedded implicit network~\eqref{eq:INN-embedding} is
        unique, can be computed explicitly using back-substitution, and corresponds exactly to
        the interval bound propagation approach in~\citep{SG-etal:18}.

      \item Motivated by the interval bound proportion approaches for
        robustness of feedforward neural networks (see for instance~\citep{SG-etal:18}), the following
        fixed point equation for estimating the output of the network
        is proposed in~\citep{CW-JZK:22}:
\begin{align*}
    \begin{bmatrix}
      \underline{x}\\
      \overline{x}
    \end{bmatrix} &= \begin{bmatrix}
      \Phi([W]^{+}\underline{z}+[W]^{-}\overline{z}+[U]^+\underline{x}
      + [U]^{-}\overline{u} + b)\\
      \Phi([W]^{+}\overline{z}+[W]^{-}\underline{z}+[U]^+\overline{x}
      + [U]^-\underline{x} + b)
    \end{bmatrix} , \\
    \begin{bmatrix}
      \underline{y}\\
      \overline{y}
    \end{bmatrix} &= \begin{bmatrix}
      [C]^{+} & [C]^{-}\\
      [C]^{-} & [C]^{+}
    \end{bmatrix}\begin{bmatrix}
      \underline{z}\\
      \overline{z}
    \end{bmatrix} + \begin{bmatrix}
      c\\
      c
    \end{bmatrix},     
\end{align*}
It is worth mentioning that the condition
$\mu_{\infty,[\eta]^{-1}}(W) < 1$ proposed in Theorem~\ref{thm:INN}
does not, in general, ensure well-posedness of the above fixed point
equation. Note that $ \lceil W \rceil^{\mathrm{Mzl}} \le [W]^{+}$ and
$[W]^-\le \lfloor W \rfloor^{\mathrm{Mzl}} $ for every $W\in
\real^{n\times n}$. As a result, compared to the inclusion
function $\begin{bmatrix}\underline{\OG}_{\ON}\\ \overline{\OG}_{\ON}\end{bmatrix}$ defined in Theorem~\ref{thm:INN}\eqref{p3:L4DC},
the above iteration provides a more
conservative estimate of the reachable sets.  
\end{enumerate}
\end{remark}

      \section{Training robust implicit neural networks}

      In this section, we design optimization problems for training implicit
      neural networks which are robust to input
      perturbations. Consider the implicit neural network~\eqref{eq:inn} and assume that
      $\{(\widehat{x}^l,\widehat{y}^l)\}_{l=1}^{N}$ is a set of $N$
      labeled data points used for training. For every
      $l\in \{1,\ldots,N\}$, we define the following upper and the
      lower bounds on the input $\widehat{x}^l$ by
      $\underline{x}^l=\widehat{x}^l-\epsilon \vect{1}_r$ and
      $\overline{x}^l=\widehat{x}^l+\epsilon \vect{1}_r$.  We use the
      robust optimization framework~\cite{AM-AM-LS-DT-AV:17} for
      designing robust neural networks. Let
      $\mathcal{L}$ be the cross-entropy loss function, then one can define
      the following robust training problem for the implicit neural network~\eqref{eq:inn}:
      \begin{align}\label{eq:training}
        \min_{W,U,C,b,c,\eta}\quad&\sum_{l=1}^{N} \max_{x^l\in
        [\underline{x}^l,\overline{x}^l]}\mathcal{L}(f_{\ON}(x^l),\widehat{y}^l), \nonumber\\
        &z^l = \ON(z^l,x^l),\quad f_{\ON}(x^l)=
        C z^l + c, \nonumber\\
        &\mu_{\infty,[\eta]^{-1}}(W)\le \gamma,
      \end{align}
      where $\gamma < 1$ is a hyperparameter ensuring
      the fixed point problem is well-posed. Unfortunately, using the
      robust loss for training in~\eqref{eq:training} leads to a
      min-max optimization problem that scales poorly with the size of
      the training data~\cite{EW-ZK:18}. In the next two paragraphs,
      we provide two relaxation of this algorithm using our estimates
      of the Lipschitz constants and the tight inclusion functions.

      \paragraph*{Lipschitz bounds.}
     Since the cross-entropy loss function is convex, there exists $\lambda>0$
     such that, for every $l\in \{1,\ldots,N\}$, 
     \begin{align*}
       \mathcal{L}(f_{\ON}(x^l),\widehat{y}^l) \le
       \mathcal{L}(f_{\ON}(\widehat{x}^l),\widehat{y}^l) + \lambda \Lip_{\infty}(f_{\ON}).
     \end{align*}
     Now using the upper bound on $\Lip_{\infty}(f_{\ON})$
     in Theorem~\ref{thm:lip}\eqref{p1:bound},  for every $l\in \{1,\ldots,N\}$, 
\begin{align*}
       \mathcal{L}(f_{\ON}(x^l),\widehat{y}^l) \le
       \mathcal{L}(f_{\ON}(\widehat{x}^l),\widehat{y}^l) + \lambda \left(\tfrac{\eta_{\max}}{\eta_{\min}}\right)\tfrac{\norm{U}{\infty}
        \norm{C}{\infty} }
          {1-\mu_{\infty,[\eta]^{-1}}(W)^+}. 
     \end{align*}
  As a result, using the Lipschitz bounds for $f_{\ON}$, we can relax
  the training optimization problem~\eqref{eq:training} to obtain the
  \emph{Lipschitz training algorithm}:
\begin{align} \label{eq:training-lip}
\min_{W,U,C,b,c,\eta}\quad&\sum_{l=1}^{N} \mathcal{L}(f_{\ON}(\widehat{x}^l),\widehat{y}^l) + \lambda \left(\tfrac{\eta_{\max}}{\eta_{\min}}\right)\tfrac{\norm{U}{\infty}
        \norm{C}{\infty} }
          {1-\mu_{\infty,[\eta]^{-1}}(W)^+}, \nonumber\\
        &z^l = \ON(z^l,x^l),\quad f_{\ON}(x^l)=
          C z^l+ c, \nonumber\\
        &\mu_{\infty,[\eta]^{-1}}(W)\le \gamma,
\end{align}
where $\lambda$ is the
regularization hyperparameter and $\gamma\in (-\infty,1)$ is the
well-posedness hyperparameter.

\paragraph*{Inclusion functions.} Following~\citep[Eq. 3]{HZ-etal:20} and~\citep{SJ-MA-AD-FB-SC:21y}, for each input $x' \in [\underline{x},\overline{x}]$, we define the \emph{relative classifier variable}, $m^x(x')\in \real^q$ by
\begin{equation}\label{eq:relativeclassifier}
    m^x(x') = f_{\ON}(x')_i \vectorones[q] - f_{\ON}(x'),
\end{equation} 
where $i$ is the correct label of $x$. Note that $m^x(x')_j > 0$ for
all $j \neq i$ if and only if $x'$ is labeled the same as $x$ by the
neural network. Therefore, we write $m^x(x') = T^x f_{\ON}(x') = T^x
Cz^*_{x'} + T^x c$, for suitable specification matrix $T^x \in
\{-1,0,1\}^{q \times q}$ defined via the linear
transformation~\eqref{eq:relativeclassifier}.  We can use Theorem~\ref{thm:INN} to define
\begin{align}\label{eq:relative-classifier}
    \underline{m}^x(\underline{x},\overline{x}) = [T^x C]^+ \underline{z}^* + [T^x C]^- \overline{z}^* + T^x c. 
\end{align}
It is clear that $\underline{m}^x(\underline{x},\overline{x})$ is a
lower bound for the relative classifier variable $m^x$. Using~\citep[Theorem 2]{EW-ZK:18},
for the cross-entropy loss, and for
$\underline{m}^l :=
\underline{m}^{\widehat{x}^l}(\underline{x}^l,\overline{x}^l)$ and
every $l\in\{1,\ldots,N\}$,
\begin{align*}
  \mathcal{L}(f_{\ON}(x^l),\widehat{y}^l) \le \mathcal{L}(-\underline{m}^l,\widehat{y}^l),\quad \text{for all } x^l\in [\underline{x}^l,\overline{x}^l].
\end{align*}
Therefore, one can instead use the loss function
$\sum_{l=1}^{N}\mathcal{L}(-\underline{m}^l,\widehat{y}^l)$ as a
tractable upper bound on the robust loss in the training optimization
problem.

As pointed out in~\cite{SG-etal:18}, using the robust loss
$\sum_{l=1}^{N}\mathcal{L}(-\underline{m}^l,\widehat{y}^l)$ in the
training can lead to convergence instability. To improve the stability
of the training, following~\cite{SG-etal:18}, we instead use a convex
combination of the empirical risk loss and the robust loss. Therefore,
for $T^l := T^{\widehat{x}^l}$, we get the 
\emph{inclusion function training algorithm}:
\begin{align}\label{eq:training-robust-2}
  \min_{W,U,C,b,c,\eta}&\quad\quad \sum_{l=1}^{N}(1-\kappa) \mathcal{L}(f_{\ON}(\widehat{x}^l),\widehat{y}^l) + \kappa \mathcal{L}(-\underline{m}^{l},\widehat{y}^l),\nonumber\\
  \begin{bmatrix}
    \underline{z}^l\\
    \overline{z}^l
  \end{bmatrix} &= \begin{bmatrix}
    \ON^{\mathrm{E}}(\underline{z}^l,\overline{z}^l,\underline{x}^l, \overline{x}^l)\\
    \ON^{\mathrm{E}}(\overline{z}^l,\underline{z}^l,\overline{x}^l,\underline{x}^l)
  \end{bmatrix}, \nonumber\\
  \underline{m}^{l} &=
                      [T^{l}C]^+\underline{z}^l+[T^{l}C]^-\overline{z}^l
                      + T^{l}c, \; \; z^l = \ON(z^l,\widehat{x}^l),
                      \nonumber \\ f_{\ON}(\widehat{x}^l)&= Cz^l + c, \quad
                                         \mu_{\infty,[\eta]^{-1}}(W)\le
                                         \gamma. 
\end{align}
where $\kappa\in [0,1]$ is the regularization parameter and
$\gamma\in (-\infty,1)$ is the well-posedness hyperparameter.

In both optimization problems~\eqref{eq:training-lip}
and~\eqref{eq:training-robust-2}, we can remove the constraint
$\mu_{\infty,[\eta]^{-1}}(W)\leq \gamma$ in the training using the
following parametrization of weight matrix $W$~\citep[Appendix
B]{SJ-AD-AVP-FB:21f}:
\begin{align}\label{eq:parametrization-weight}
  W = [\eta]^{-1}T[\eta] - \diag(|T|\vect{1}_n) +\gamma I_n,
\end{align}
for an unconstrained matrix $T\in \real^{n\times n}$. Using the
parametrization~\eqref{eq:parametrization-weight} in the training
problem not only improves the computational efficiency of the
optimization but also allows for the design of implicit neural
networks with additional structure such as convolutions. We refer
to~\citep{SJ-AD-AVP-FB:21f} for more details about training of
convolutional implicit neural networks.

\section{Theoretical and numerical comparisons}
In this section, we first introduce the notion of
certified adversarial robustness for classification problems.
We say that the implicit neural network~\eqref{eq:inn} is \emph{certified
  adversarially robust} with radius $\epsilon$ at input $x$ if
\begin{multline}
  \max_{v\in\real^r} \; \setdef{f_{\ON}(x')_i-\max_{j\ne
  i}f_{\ON}(x')_j}{\|x-x'\|_{\infty}\le \epsilon,\;\;\nonumber \\ i
\;\;\mbox{ is the correct label of }x} \ge 0.
  \end{multline}
  Verification of certified adversarial robustness can be computationally complicated. In the next
  two paragraphs, we use the frameworks in
  Section~\ref{sec:reach} to provide lower bounds for certified
  adversarial robustness.

\paragraph*{Lipschitz bounds.}
Consider an implicit neural network~\eqref{eq:inn} with $\eta\in
\real^n_{>0}$ such that $\mu_{\infty,[\eta]^{-1}}(W)<1$. Using Theorem~\ref{thm:lip}\eqref{p2:interval}, if
\begin{multline}\label{eq:CAR-lip}
    f_{\ON}(x)_i-\max_{j\ne
    i}f_{\ON}(x)_j \\
    -2 \left(\frac{\eta_{\max}}{\eta_{\min}}\right)
    \frac{\norm{U}{\infty}
    \norm{C}{\infty} \epsilon }
    {1-\mu_{\infty,[\eta]^{-1}}(W)^+}\ge 0 
\end{multline}
holds, then the implicit neural network~\eqref{eq:inn} is certified adversarially
robust with radius $\epsilon$.

\paragraph*{Inclusion functions.} Consider an implicit neural
network~\eqref{eq:inn} with $\eta\in
\real^n_{>0}$ such that $\mu_{\infty,[\eta]^{-1}}(W)<1$. Note that
$\underline{m}^x(\underline{x},\overline{x})$ defined in~\eqref{eq:relative-classifier} is a
lower bound for the relative classifier variable $m^x$. Thus, one can use Theorem~\ref{thm:INN} to show that, if  $i$ is the correct label of the input $x$ and
\begin{align}\label{eq:CAR-inclusion}
  \min_{j \neq i} \{\underline{m}_j^x(x-\epsilon\vect{1}_r,
  x+\epsilon\vect{1}_r)\}\ge 0
\end{align}
holds, then the implicit neural network~\eqref{eq:inn} is
certified adversarially robust with radius $\epsilon$.

      \subsection{MNIST experiments}
      In this section, we compare the certified adversarial robustness
      of different approaches on the MNIST handwritten digit dataset,
      a dataset of $70000$ $28 \times 28$ pixel images, $60000$ of
      which are for training, and $10000$ for testing. Pixel values
      are normalized in $[0,1]$\footnote{All experiments were run on a Tesla P100-PCIE-16GB GPU in Google Colab}.

      In the first experiment, we train an
      implicit neural network with $n=100$ neurons using the Lipschitz training
      algorithm~\eqref{eq:training-lip} for different values of
      $\lambda\in \{0,10^{-1}, 10^{-2}, 10^{-3}, 10^{-4}, 10^{-5}\}$. For well-posedness, we
      imposed $\mu_{\infty,[\eta]^{-1}}(W) \leq 0$ and we directly parameterize $W$ as in~\eqref{eq:parametrization-weight}. Training data
      is broken up into batches of $100$ and the model was trained
      for $15$ epochs with a learning rate of $10^{-3}$. Regarding training times, using the $\alpha$-average iteration in Theorem~\ref{thm:inn-well-posed}\eqref{p2:average}, our model takes on average 9.8 seconds to train per epoch. After
      training, the models are validated on test data using the
      sufficient conditions for certified adversarial
      robustness~\eqref{eq:CAR-lip}. For fixed $\epsilon$ and the $10000$ test
      images, over $10$ trials, it takes, on average, $2.250$ seconds
      to verify the certified adversarial robustness using the
      formula~\eqref{eq:CAR-lip}. To
      provide a conservative upper bound on the certified adversarial
      robustness and to observe empirical robustness, the model was
      additionally attacked using projected gradient descent (PGD) and
      fast-gradient sign method (FGSM) attacks. Results from these
      experiments are shown in Figure~\ref{fig:NEMON-MNIST}. We compare robustness of our implicit neural networks with the Monotone Operator Deep Equilibrium (MON) model~\citep{EW-JZK:20} with the monotonicity parameter $m=1$.
      \begin{figure}[!ht]
	\begin{center}
	    \includegraphics[width = 0.92\linewidth,clip]{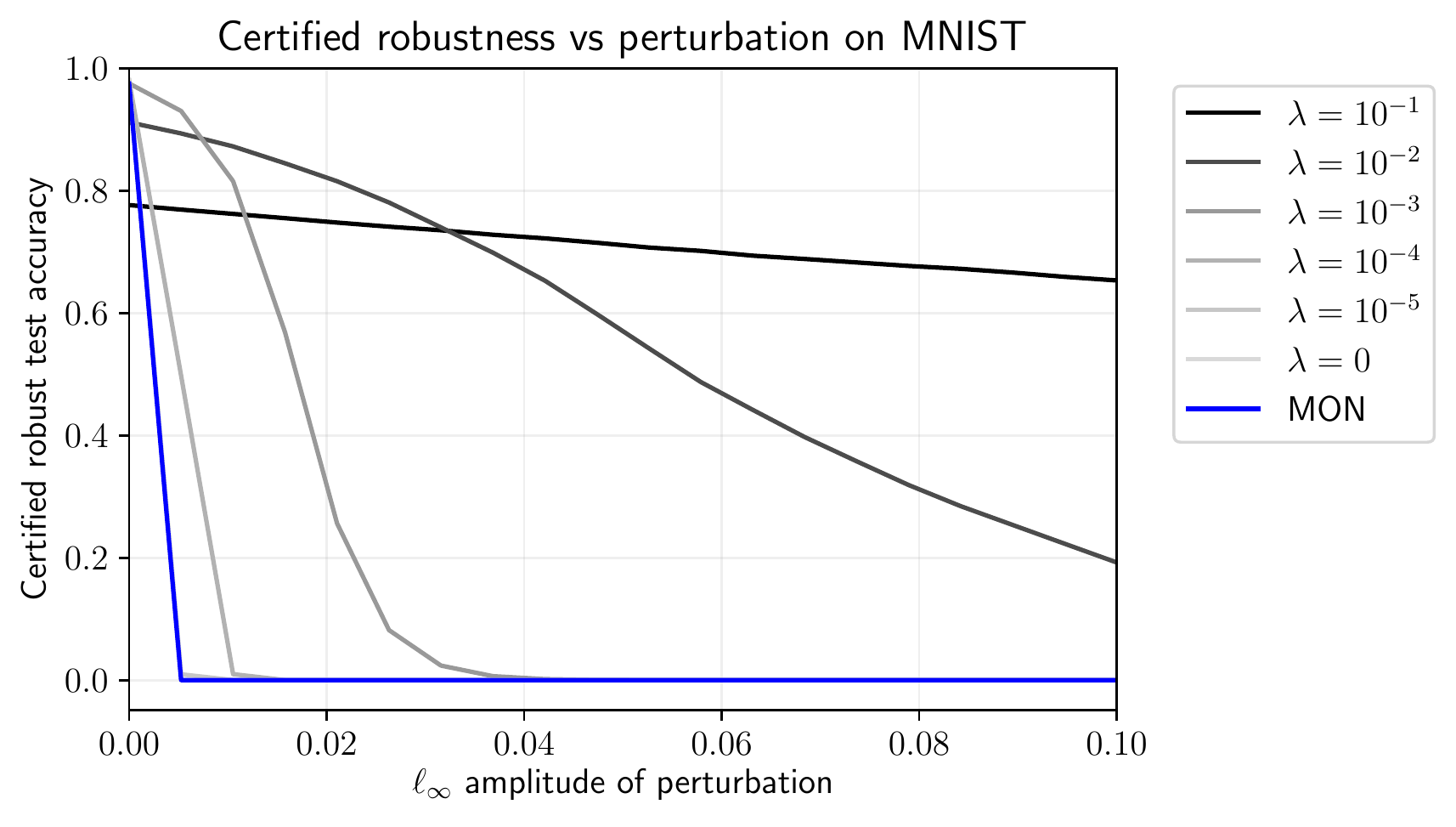}\\
		\hspace{-1.6cm}\includegraphics[width = 0.75\linewidth,clip]{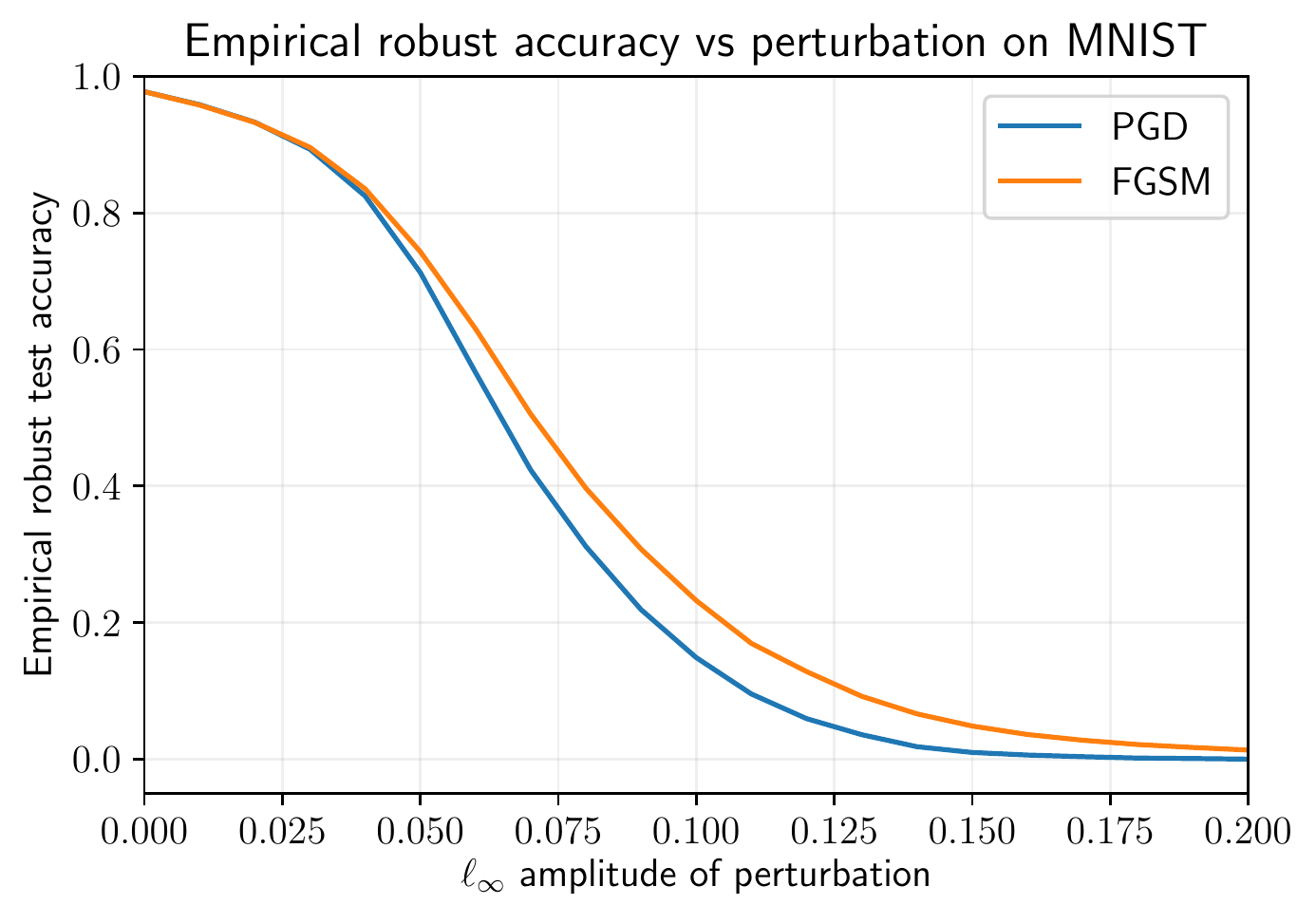}
	\end{center}
\vspace{-.35cm}
    \caption{On the top is a plot of the certified adversarial
      robustness of the models trained using Lipschitz training algorithm~\eqref{eq:training-lip} for
      different regularization hyperparameter $\lambda$. The top plot include the certified adversarial
      robustness of the MON model trained with the monotonicity parameter $m=1$. For fixed
      $\epsilon$, the fraction of test inputs which are certified
      robust are plotted. At the bottom is a plot of the empirical
      robustness of the implicit neural model trained with
      the Lipschitz training algorithm~\eqref{eq:training-lip} with the regularization hyperparameter $\lambda=10^{-5}$ subject to the PGD and the FGSM attacks. Note the difference in scale on the horizontal axis.}
    \label{fig:NEMON-MNIST}
\end{figure}

In the second experiment, we train two implicit neural networks using
the inclusion function training algorithm~\eqref{eq:training-robust-2}. For
well-posedness, we impose $\mu_{\infty,[\eta]^{-1}}(W) \leq 0$ for
some $\eta \in \real^n_{>0}$ and we directly parameterize $W$ as in~\eqref{eq:parametrization-weight}. Both models are trained for $40$ epochs
using the Adam optimizer and a learning rate of $5 \times 10^{-4}$. At epoch $30$,
the learning rate is decreased to $10^{-4}$. For the first model (shown
by dashed lines in Figure~\ref{fig:MNIST}), we set $\epsilon=\kappa=0$
during the training. This is equivalent to training a
non-robust implicit neural network. For the second model (shown by solid lines in
Figure~\ref{fig:MNIST}), we pick $\subscr{\epsilon}{test} = 0.1$ and $\subscr{\kappa}{nom} =
0.75$. From epochs $1$ to $10$, $\kappa$ and $\epsilon$ are set to $0$
so the models undergo regular (non-robust) training. From epochs $11$
to $20$, $\epsilon$ and $\kappa$ are linearly increased such that at
epoch $20$, $\epsilon = \subscr{\epsilon}{test}$ and
$\kappa = \subscr{\kappa}{nom}$. Regarding the training time, using the $\alpha$-average iterations in Theorem~\ref{thm:INN}\eqref{p1:L4DC}, the second model takes on average 23.9 seconds to train per epoch. After training, both models are validated on test data using the sufficient conditions for certified adversarial robustness~\eqref{eq:CAR-inclusion}. For a fixed $\epsilon$ and  over the 10000 test images over 10 trials, on average, it takes 11.29 seconds to compute the certified adversarial robustness using formula~\eqref{eq:CAR-inclusion}. Figure~\ref{fig:MNIST} provides plots for this experiment.

      \begin{figure}[!ht]
 \begin{center}
		\includegraphics[width = 0.75\linewidth,clip]{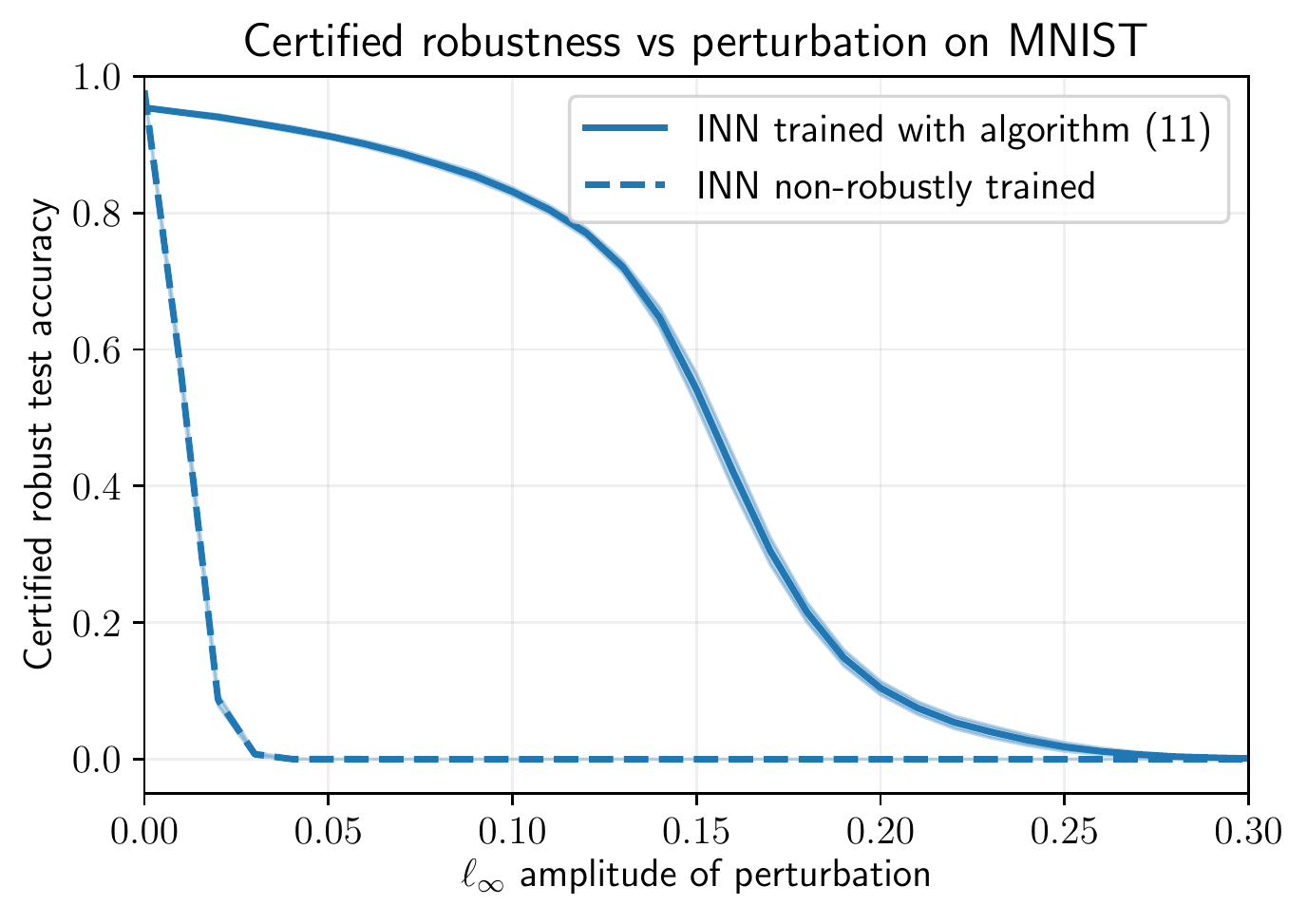}\\
		\includegraphics[width = 0.75\linewidth,clip]{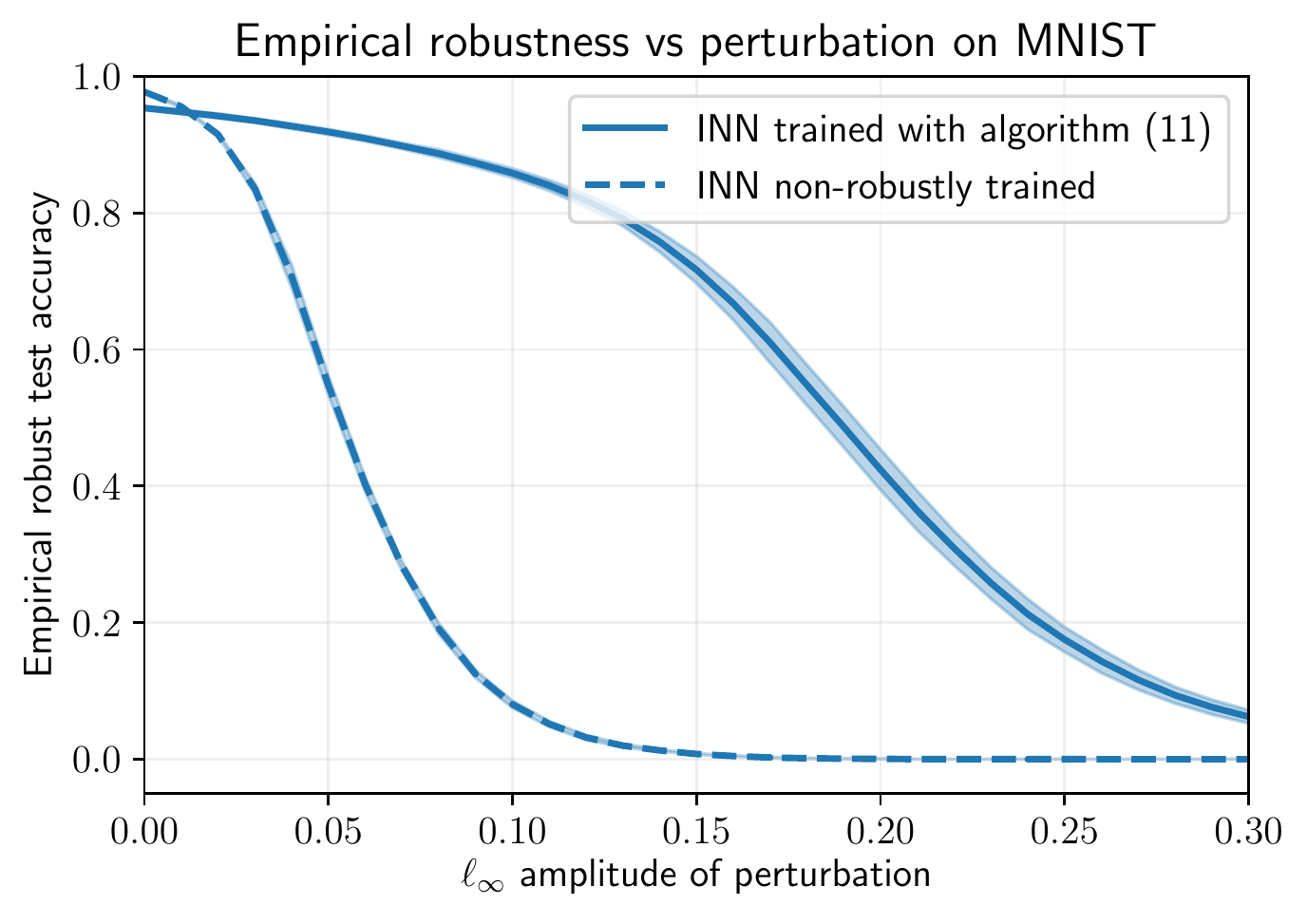}
	\end{center}
	\caption{On the top is a plot of the certified adversarial
      robustness of the trained model using inclusion function training
      algorithm~\eqref{eq:training-robust-2}. At the bottom is a plot of the empirical
      robustness of the implicit neural model trained with
      inclusion function training algorithm~\eqref{eq:training-robust-2} subject to the PGD attacks.}
      \vspace{-.35cm}
	\label{fig:MNIST}
\end{figure}

      \paragraph*{Summary evaluation.} From the first experiment, we
      can study the role of the regularization hyperparameter $\lambda$ in
      the Lipschitz training algorithm~\eqref{eq:training-lip}. From Figure~\ref{fig:NEMON-MNIST}
      it is clear that increasing the value of $\lambda$ leads to increased
      certified robustness of the model. However, this
      increase is obtained at the cost of reduction in clean accuracy. 
      Moreover, compared to the MON model~\citep{EW-JZK:20}, our implicit models with regularization parameter larger than $10^{-5}$ are certifiably more robust with respect to sizable $\ell_{\infty}$-norm input perturbations. 
      Finally, by
      comparing the top plot and the bottom plot in Figure~\ref{fig:MNIST},  one can see a very large gap between the
      certified adversarial robustness and the empirical robustness under both PGD and FGSM attacks.

      From the second experiment, we can conclude that implicit neural
      networks trained using the
      inclusion function training algorithm~\eqref{eq:training-robust-2} (solid lines) vastly outperform
      the non-robustly trained models (the dashed line) in both
      certified and empirical robustness. For instance, at an
      $\ell_\infty$ perturbation radius of $0.1$, we observe that the
      model trained using the inclusion function training algorithm, on average, has certified
      robustness of $83.13\%$ and empirical robustness of $85.84\%$
      with respect to PGD attack. It is worth mentioning that, at an
      $\ell_\infty$ perturbation radius of $0.1$, the non-robustly trained model has
      $0 \%$ certified robustness and $8.04\%$ empirical robustness 
      with respect to PGD attack.

      Finally, we compare the performance of the model trained using the
      Lipschitz optimization algorithm~\eqref{eq:training-lip} in Figure~\ref{fig:NEMON-MNIST}
      with the model trained using the inclusion function
      optimization algorithm~\eqref{eq:training-robust-2} in
      Figure~\ref{fig:MNIST}. We can deduce that (i) the Lipschitz bound approach is significantly faster in both training the models and verification of their certified adversarial robustness and (ii) 
      the inclusion function approach will lead to more accurate and more robust models.  

      \section{Conclusions}
      
      Using non-Euclidean contraction theory, we develop a framework for studying robustness of implicit neural networks. For a given implicit neural network, we use estimates of (i) its $\ell_{\infty}$-norm input-output Lipschitz constant, and (ii) its tight inclusion function to obtain $\ell_{\infty}$-norm box upper bounds for input-output behavior of the network. Based on these upper bounds, we design two algorithms for training and robustness verification of implicit neural networks. Empirical evidence shows the efficiency of our algorithms. 
      
      \section*{Acknowledgement}
      This work was supported in part by Air Force Office of Scientific Research under grants FA9550-22-1-0059 and FA9550-19-1-0015 and National Science Foundation under grant $\#$1836932.

\bibliography{alias,FB,main2}
\bibliographystyle{wfvml2022}

\end{document}